\documentclass[nohyperref, twocolumn]{article}
\usepackage[centering, margin={.75in, 0.5in}, includeheadfoot]{geometry}
\usepackage{microtype}
\usepackage{graphicx}
\usepackage{subfigure}
\usepackage{booktabs}
\usepackage{hyperref}

\usepackage{amsmath}
\usepackage{amssymb}
\usepackage{mathtools}
\usepackage{amsthm}
\usepackage{cancel}

\usepackage{natbib}

\usepackage[capitalize,noabbrev]{cleveref}

\DeclareMathOperator{\Ex}{\mathbb{E}}
\DeclareMathOperator{\Var}{\text{Var}}

\DeclareMathOperator*{\argmax}{arg\,max}
\DeclareMathOperator*{\argmin}{arg\,min}

\newcommand{\calA}{\mathcal{A}}

\newcommand{\calD}{\mathcal{D}}
\newcommand{\calX}{\mathcal{X}}

\newcommand{\calR}{\mathcal{R}}

\newcommand{\IPS}{\text{IPS}}
\newcommand{\BAL}{\text{BAL}}

\newcommand{\x}{{x}}                 
\newcommand{\xRV}{{X}}               
\newcommand{\action}{{a}}                 
\newcommand{\actionRV}{{A}}               
\newcommand{\reward}{{r}}                 
\newcommand{\Util}{{U}}                 
\newcommand{\UtilIPS}{{\hat{U}^{\text{IPS}}}}                 

\newcommand{\old}{{log}}
\newcommand{\eval}{{aug}}
\newcommand{\target}{{tar}}

\newcommand{\policy}{\pi}
\newcommand{\policyspace}{\Pi}
\newcommand{\oldpolicy}{\pi_{\text{\old}}}
\newcommand{\targetpolicy}{\pi_{\text{\target}}}
\newcommand{\targetpolicyspace}{\policyspace_{\text{\target}}}
\newcommand{\evalpolicy}{\pi_{\text{\eval}}}
\newcommand{\blendedpolicy}{\pi_{\text{balanced}}}
\newcommand{\minvarpolicy}{\pi_{\text{minvar}}^{\text{IPS}}}

\newcommand{\maxpolicy}{\pi_{\text{max}}}

\newcommand{\paren}[1]{\left(#1\right)}
\newcommand{\sqaren}[1]{\left[#1\right]}

\newcommand{\dataset}{\calD}
\newcommand{\evaln}{{n_\text{\eval}}}
\newcommand{\oldn}{{n_\text{\old}}}
\newcommand{\evalD}{{\dataset_\text{\eval}}}
\newcommand{\oldD}{{\dataset_\text{\old}}}

\newcommand{\rewardmean}{{\bar \reward(\x,\action)}}
\newcommand{\rewardmeansq}{{\bar \reward^2(\x,\action)}}
\newcommand{\rewardvar}{{\sigma^2(\x,\action)}}
\newcommand{\expectedsquarer}{\Ex_r\sqaren{r^2(x, a)}}
\newcommand{\expectedsquarertight}{\Ex_r\!\!\sqaren{r^2(x, a)}}

\newcommand{\rewardmeansqprewardvar}{\expectedsquarer}
\newcommand{\parenrewardmeansqprewardvar}{\expectedsquarer}

\newtheorem{optimizationproblem}{Optimization Problem}

\theoremstyle{plain}
\newtheorem{theorem}{Theorem}[section]
\newtheorem{proposition}[theorem]{Proposition}
\newtheorem{lemma}[theorem]{Lemma}

\theoremstyle{definition}

\theoremstyle{remark}

\usepackage[textsize=tiny]{todonotes}

\title{Variance-Optimal Augmentation Logging for Counterfactual Evaluation in Contextual Bandits}

\author{Aaron David Tucker\footnote{aarondtucker@cs.cornell.edu}\hspace{4pt} and Thorsten Joachims\\ 
Department of Computer Science, Cornell University, Ithaca NY\\
}

\begin{document}

\maketitle

\begin{abstract}
Methods for offline A/B testing and counterfactual learning are seeing rapid adoption in search and recommender systems, since they allow efficient reuse of existing log data. However, there are fundamental limits to using existing log data alone, since the counterfactual estimators that are commonly used in these methods can have large bias and large variance when the logging policy is very different from the target policy being evaluated. To overcome this limitation, we explore the question of how to design data-gathering policies that most effectively augment an existing dataset of bandit feedback with additional observations for both learning and evaluation. To this effect, this paper introduces Minimum Variance Augmentation Logging (MVAL), a method for constructing logging policies that minimize the variance of the downstream evaluation or learning problem. We explore multiple approaches to computing MVAL policies efficiently, and find that they can be substantially more effective in decreasing the variance of an estimator than naïve approaches.
\end{abstract}

\section{Introduction}
Logged user feedback from online systems is one of the primary sources of training data for search and recommender systems. However, learning from log data is challenging, since the rewards are only partially observed.
In particular, logged data contains the observed reward (e.g.\ click/no click) only for the specific action (e.g.\ movie recommendation) that the historic system took, but it does not include the reward observations for the other possible actions the system could have taken (e.g. all other movies). 
This means that offline policy evaluation requires us to deal with counterfactual outcomes when the historic system and the new policy do not chose the same action.
Fortunately, over the recent years an increasingly rich set of counterfactual estimators \cite{bottoucounterfactual,doubly,morerobust,selfnorm,magic,switch,cab,defsupport} and learning methods \cite{beygelzimer2009offset,strehl2010learning,poem,banditnet,cab,london2019bayesian} have been developed that can use logged user feedback with strong theoretical guarantees despite the partial nature of the data. 

These developments have led to increasing adoption of counterfactual learning and evaluation in real-world applications, where they are used to conduct unbiased ``offline A/B tests" and to train policies that directly and provably optimize online metrics. However, we point out that all counterfactual methods are fundamentally limited by the information contained in the logs. In particular, if the policy that logged the data is substantially different from the target policy we want to evaluate, the variance of the estimate will be high or the estimator may even be biased \citep[e.g.,][]{defsupport}.


To overcome this fundamental limitation of counterfactual methods, this paper explores how to best collect a limited amount of additional log data to maximize the effectiveness of the counterfactual estimator. We call this the problem of {\em augmentation logging}, and we study how to design augmentation logging policies that optimally augment an existing log dataset. The resulting methods can be used to optimize data efficiency for A/B testing, and even address the question of how to log data in contextual-bandit systems that are re-trained periodically (e.g., weekly).

The main contributions of this paper are four-fold. First, we introduce and formalize the problem of augmentation logging as minimizing the bias and variance of the counterfactual estimator given existing logged data. Second, we show how to compute variance-optimal augmentation policies and provide a theoretical characterization of this approach. Third, we develop a method to approximate the optimal augmentation policy, improving its online efficiency. Finally, we empirically evaluate the methods for both counterfactual policy evaluation and counterfactual policy learning.


\section{Related Work}
While our formulation of optimal augmentation logging for counterfactual policy evaluation and learning is new, it is connected to several bodies of existing literature.

\vspace*{-1mm}\paragraph{Off-policy evaluation.}
Counterfactual evaluation (a.k.a. off-policy evaluation, offline evaluation, or offline A/B testing) has been widely studied as a method of estimating the value of a policy based on data collected under a different policy. Estimators like Inverse Propensity Score (IPS) weighting and its variants \cite{ips,strehl2010learning,selfnorm} are typically used to correct the distribution shift between logging and target policy.
While unbiased under full support \cite{owen,strehl2010learning}, such estimators can have large variance.
Our method aims to reduce the variance of the estimator, which is also the motivation behind work such as \citet{bottoucounterfactual,doubly,morerobust,cab,defsupport,switch}, or \citet{magic} for evaluation; and behind \citet{bottoucounterfactual,strehl2010learning,poem,banditnet,cab}, or \citet{london2019bayesian} for learning. However, instead of treating the data as fixed and trying to reduce bias and variance with this constraint, we investigate which additional data would most reduce bias and variance. 

\vspace*{-1mm}\paragraph{Online Contextual Bandits.} Online contextual bandits (see e.g. \citet{cesa-bianchi06}) have been studied as a model of sequential decision making under a variety of settings and modeling assumptions \cite{contextualbandits,lints,monster,bakeoff,news}. Our work shares the bandit feedback and contextual features of the online contextual bandit setting, but it relaxes the requirements of real-time adaptivity that defines online algorithms. In particular, we find that many real-world settings do not allow for a centralized and synchronous control that permits constant adaptation of the policy -- and the risks to robustness and adversarial manipulation that come with that. Instead, our setting of infrequent batch updates to the policy allows the easy use of cross-validation for hyperparameter tuning, as well as the opportunity to carefully validate any policy before it gets fielded, which greatly improves practical feasibility.


\vspace*{-1mm}\paragraph{Monte-Carlo Estimation.} Designing optimal sampling distributions is a problem widely considered in Monte-Carlo estimation \cite{owen}. We draw upon foundational results about which sampling strategies are optimal for importance sampling estimators, which are analogous to IPS estimators. Moreover, we relate augmentation logging to multiple importance sampling \cite{balancedestimator}, and we show how to extend these methods to get uniform bounds on the variance of a class of target policies.

\section{Single Policy Evaluation}

We begin by formalizing the augmentation-logging problem for evaluating a single target policy $\targetpolicy$, and then extend this approach to multi-policy evaluation and learning in Section~\ref{sec:multipol}. In all three settings, we consider contextual bandit policies, which are widely used to model search and recommendation problems \cite{news, search}. At each time step $i$, a context $\x_i$ (e.g., query, user request) is sampled i.i.d.\ from an underlying distribution $\x_i \sim \Pr(\xRV)$, and a policy $\policy$ stochastically chooses an action $\action_i$ (e.g., a movie to recommend) such that $\action_i \sim \policy(\actionRV|\x)$. The system then observes the reward $\reward_i$ (e.g., purchase) for action $\action_i$ from the environment. 

The central question in single-policy evaluation lies in estimating the expected reward (a.k.a.\ utility)
\begin{equation}
    \Util(\targetpolicy) = \sum_\x \sum_\action \Ex_r[r(\x,\action)] \targetpolicy(\action|\x) \Pr(\x)
\end{equation}
of some target policy $\targetpolicy$. The conventional approach is to field this target policy in an A/B test, which allows us to estimate $\Util(\targetpolicy)$ simply from the average of the observed rewards. However, such online A/B tests typically take a long time to complete, and they do not scale when we need to evaluate many target policies. Therefore, offline evaluation has seen substantial interest, since it computes an estimate of $\Util(\targetpolicy)$ using only historic data
\begin{equation}
    \oldD = \{x_i, a_i, r_i\}_{i=1}^\oldn 
\end{equation}
already logged from some other policy $\oldpolicy$.\footnote{For simplicity of notation, we assume all observations were collected from the same $\oldpolicy$. However, all results in this paper can be extended to the case where the data comes from multiple logging policies.} The key challenge lies in the fact that the logging policy $\oldpolicy$ is typically picks actions that are different from those selected by the target policy $\targetpolicy$. This challenge can be addressed by counterfactual estimators such as inverse propensity score (IPS) weighting
\begin{equation}
    \UtilIPS(\targetpolicy) = \sum_{i=1}^{\oldn} \frac{\targetpolicy(\action_i|\x_i)}{\oldpolicy(\action_i|\x_i)} \reward_i.
\end{equation}
The IPS estimator can be shown to be unbiased whenever the logging policy has full support under the target policy, namely when $\forall \x \forall \action: \targetpolicy(\action|\x) P(\x) > 0 \rightarrow \oldpolicy(\action|\x)>0$. Unfortunately, this condition is frequently violated in practical applications. Further, even if the condition is met, the IPS estimator can have excessive variance. Much work has gone into mitigating both the bias problem \cite{doubly,defsupport} and the variance problem \cite{bottoucounterfactual,doubly,cab,switch,magic}, but any estimator that only has the information in $\oldD$ is fundamentally limited.

To overcome these fundamental limits, we allow that we can augment $\oldD$ with $\evaln$ additional observations 
\begin{equation}
    \evalD = \{x_i, a_i, r_i\}_{i=1}^\evaln 
\end{equation}
from an augmentation logging policy $\evalpolicy$. We next address the key questions of which counterfactual estimator to use, and how to design $\evalpolicy$ so that the $\evaln$ additional observations most improve the quality of the utility estimate.

\subsection{Designing Variance-Optimal Augmentation Policies}

In order to reason about which augmentation policy $\evalpolicy$ minimizes bias and variance, we first need to select an estimator. As we will justify below, we focus on the balanced estimator \cite{balancedestimator},
\begin{equation}
\hat R_{\targetpolicy}^\text{BAL} = \dfrac{1}{N}\paren{\sum_{(\x_i, \action_i, \reward_i) \in \oldD \cup \evalD}\dfrac{\targetpolicy(\action_i|\x_i)}{\blendedpolicy(\action_i|\x_i)} \reward_i}, \label{equation:balanced}
\end{equation}
where $N=\oldn + \evaln$, $\alpha = \evaln / N$ and
\begin{eqnarray*}
\blendedpolicy(\action|\x) &=& (1-\alpha)\oldpolicy(\action|\x) + \alpha\evalpolicy(\action|\x).
\end{eqnarray*}
This estimator was shown to never have larger variance than the following more naïve IPS estimator 
\begin{equation*}
\hat R_{\targetpolicy}^{\IPS}
= \dfrac{1}{N}\!\!\!\!\!\!\!\!\!\sum_{\vspace{24pt}{(\x, \action, \reward)_i \in \oldD}}\!\!\!\! \dfrac{\targetpolicy(a_i|x_i)}{\oldpolicy(a_i|x_i)}r_i
+ \dfrac{1}{N}\!\!\!\!\!\!\!\!\!\sum_{\vspace{24pt}{(\x, \action, \reward)_j \in \evalD}}\!\!\!\! \dfrac{\targetpolicy(a_j|x_j)}{\evalpolicy(a_j|x_j)}r_j
\end{equation*}
that weights each action by the policy that selected it, and it can have substantially smaller variance \cite{balancedestimator}. Further, it is easy to verify that the balanced estimator is unbiased under strictly weaker conditions than the IPS estimator. In particular, the balanced estimator is already unbiased if $\forall \x \forall \action: \targetpolicy(\action|\x) P(\x)>0 \rightarrow (\oldpolicy(\action|\x)>0 \vee \evalpolicy(\action|\x)>0)$.
The balanced estimator has the following variance, with a proof provided in Appendix \ref{proof:balancedvar}.

\begin{equation}
\label{equation:balancedvar}
\text{\!\!\!Var}\!\left[\!\hat R^{\text{BAL}}_{\targetpolicy}\!\right]
\!=\! \dfrac{1}{N}\!\paren{\!\mathbb{E}_{\x}\!\!\left[\!\sum_{\action \in \calA}\!\!\!
\dfrac{\targetpolicy^2(\action|\x)\!\expectedsquarertight}{\blendedpolicy(\action|\x)}
\!\right] \!-\!
R_{\targetpolicy}^2\!\!}\!\!\!
\end{equation}

In this equation, $\expectedsquarer = \rewardmeansq+\rewardvar$, where $\rewardmean$ and $\rewardvar$ are the expected rewards and their variance conditioned on the given $x, a$. Importantly, this variance depends directly on $\blendedpolicy$, which is based on $\oldpolicy$ and $\evalpolicy$. This allows us to design a $\evalpolicy$ to compensate for high variance terms caused by $\oldpolicy$. This is not possible for the  naïve IPS estimator $\hat R_{\targetpolicy}^{\IPS}$ for the following reason.
\begin{proposition}
For the IPS estimator $\hat R_{\targetpolicy}^{\IPS}$, the variance-optimal augmentation policy is independent of $\oldpolicy$. 
\end{proposition}
{\sc Proof Sketch.} The variance of the IPS estimator is additive in the contributions from $\oldpolicy$ or $\evalpolicy$ for each logged data point, and thus the variance-minimizing $\evalpolicy$ is independent of $\oldpolicy$. This means that the variance-optimal policy would be the same policy as when $\evalpolicy$ is the only policy observed, and cannot account for anything related to the previous policy $\oldpolicy$. Details in Appendix \ref{proof:ipsnogo}. \qedsymbol{} \vspace*{0.2cm}

Given these deficiencies of the IPS estimator, we thus focus on the balanced estimator for designing augmentation logging policies. Note that it is straightforward to extend the balanced estimator to a doubly-robust setting \cite{doubly}, which we skip for the sake of brevity. We formulate the search for the variance-optimal augmentation policy as the following optimization problem, which we refer to as Minimum Variance Augmentation Logging (MVAL).
\begin{optimizationproblem}[MVAL for Single-Policy Evaluation] \label{op:mval}
For a given context $x \in \calX$,
\begin{align*}
\evalpolicy(\actionRV|x) = \argmin_{\policy \in \calR^{|\calA|}} \quad & \sum_{\action \in \calA}\dfrac{\targetpolicy^2(\action|\x)\expectedsquarer}{(1-\alpha)\oldpolicy(\action|\x) + \alpha \policy(\action)}\\
\textrm{subject to} \quad& \sum_{a\in \calA}\pi(a) = 1, \\
&\forall a \in \calA: \pi(a) \geq 0   
\end{align*}
\end{optimizationproblem}

The number of real-valued parameters $\policy(\action)$ in this optimization problem always equals the number of actions for the given context $\x$, and the optimization problem is always convex (Appendix \ref{proof:convex}). First, we note that the augmentation policy $\evalpolicy(\actionRV|x)$ computed by OP\ref{op:mval} is variance optimal.
\begin{theorem}[Variance Optimality of MVAL]
The solution of OP\ref{op:mval} minimizes the variance of the balance estimator.
\end{theorem}
\begin{proof} Follows immediately for a fixed $x$ from Equation~\eqref{equation:balancedvar}, since $\blendedpolicy(\action|\x) = (1-\alpha)\oldpolicy(\action|\x) + \alpha \policy(\action)$. Variance is a weighted sum of these independent minimized terms. \end{proof}

Second, with augmentation data from $\evalpolicy$ the balanced estimator is unbiased, even if the logging policy $\oldpolicy$ is support deficient and would otherwise lead to biased estimates.
\begin{theorem}[MVAL Guarantees Unbiasedness]
\label{thm:unbiased}
OP\ref{op:mval} always produces augmentation policies $\evalpolicy$ so that the balanced estimator is unbiased for any $\oldpolicy$ and for any choice of $\rewardmean$ and $\rewardvar>0$, even if $\rewardmean$ and $\rewardvar$ are inaccurate.
\end{theorem}


\begin{proof}
As shown in \cite{balancedestimator}, the balanced estimator is unbiased when $\blendedpolicy$ has full support for $\targetpolicy$, specifically $\forall \x \forall \action: \targetpolicy(\action|\x) P(\x)>0 \rightarrow \blendedpolicy(\action|\x)>0$. Since $\blendedpolicy$ is a convex combination of $\oldpolicy$ and $\evalpolicy$, full support is already guaranteed if $\forall \x \forall \action: \targetpolicy(\action|\x) P(\x)>0 \rightarrow (\oldpolicy(\action|\x)>0 \vee \evalpolicy(\action|\x)>0)$. To make sure that this condition is always fulfilled, we need to verify that $\evalpolicy(\action|\x)>0$ when $\oldpolicy(\action|\x)=0$. To verify this condition, note that $\oldpolicy(\action|\x)=0$ implies that the term
\begin{equation*}
    \dfrac{\targetpolicy^2(\action|\x)\expectedsquarer}{\alpha \evalpolicy(\action)}
\end{equation*}
occurs in the objective of OP\ref{op:mval}. Note that the solution of OP\ref{op:mval} cannot have $\evalpolicy(\action)=0$, since this would lead to an infinite objective which is not optimalm since the uniform $\policy$ is feasible and has a better objective.
\end{proof}

\paragraph{Computing $\rewardmeansqprewardvar$.}
Finally, we need to resolve the problem that $\expectedsquarer$ is typically unknown, for which we explore two options.

The first option is to optimize the following variant of the MVAL optimization problem, where we simply drop $\expectedsquarer$ from the objective. This is equivalent to optimizing an upper bound on the variance, where $\expectedsquarer$ is replaced with $\max_{x,a}\expectedsquarer$ in Equation~\eqref{equation:balancedvar}.

The second option is to use a regression estimate to impute the estimated value for $\expectedsquarer$. Virtually any real-valued regression technique can be applied to $\oldD$ to estimate $r^2(x, a)$, and even imperfect estimates can provide useful information about $\text{Var}[\hat R^{\text{BAL}}_{\targetpolicy}]$. Note that Theorem~\ref{thm:unbiased} holds even for incorrect estimates of $\expectedsquarer$, so that even bad reward estimates can never introduce bias.

\subsection{Analysis and Discussion}

We now further analyze the properties of MVAL policies and provide intuition through some illustrative edge cases.

\subsubsection{MVAL without Historic Log Data}

When the historic data $\oldD$ is empty, the following shows that the MVAL policy $\evalpolicy^{\BAL}(a|x)$ computed by OP\ref{op:mval} coincides with the variance-optimal logging policy $\minvarpolicy(a|x)$ for the IPS estimator.

\begin{proposition}\label{thm:minvar}
If $\alpha = \evaln / (\evaln + \oldn) = 1$, then the optimal augmentation policy is
\begin{align}\label{equation:ipsminvar}
\evalpolicy^{\BAL}(a|x) 
&= \dfrac{\targetpolicy (a|x)\sqrt{
\rewardmeansqprewardvar
}}{\sum\limits_{a\in\calA} \targetpolicy (a|x)\sqrt{
\rewardmeansqprewardvar
}}\\&=  \minvarpolicy(a|x).
\end{align}
\end{proposition}
{\sc Proof Sketch.} 
This can be shown using Lagrange multipliers to solve OP\ref{op:mval}, and using the well known result \cite{owen} characterizing the variance-optimal IPS logging policy. Detailed proof is in Appendix \ref{proof:minvar}.\qedsymbol{}

This immediately provides a closed form solution for MVAL in an intuitive special case. 
\begin{proposition}[Target policy optimality without logged data]\label{thm:target} If $\alpha=1$ and $\rewardmeansqprewardvar= c$, then the optimal augmentation policy for single policy evaluation using the balanced estimator is the target policy.\end{proposition}
\begin{proof}This is a consequence of Equation \ref{equation:ipsminvar}, which states that if $\alpha=1$, then the optimal augmentation policy $\evalpolicy^{\BAL}(a|x) \propto \targetpolicy(a|x)\sqrt{\rewardmeansqprewardvar}$. If $\rewardmeansqprewardvar = c$, then it follows that $\evalpolicy^{\BAL}(a|x) \propto \targetpolicy(a|x)$, and therefore $\evalpolicy^{\BAL}(a|x) = \targetpolicy(a|x)$.\end{proof}

\subsubsection{MVAL Corrects Historic Log Data Towards the Target Policy} 

When the historic data is non-empty, the augmentation policy $\evalpolicy$ minimizes the balanced estimator variance by trying to make the balanced policy $\blendedpolicy$ similar to the minimum variance IPS policy $\minvarpolicy$. Specifically, when there is enough augmentation data for  $\evalpolicy$ to cause $\blendedpolicy=\minvarpolicy$, then that is the solution chosen. In this case, we can even compute the MVAL policy in closed form.

\begin{proposition}[Large-$\alpha$ Closed-Form Solution for MVAL]
\label{thm:bigalpha}
If $\alpha$ is big enough that the mixed policy $\blendedpolicy = (1-\alpha)\oldpolicy + \alpha \evalpolicy$ can be equal to the minimum variance augmentation policy for the IPS estimator $\minvarpolicy$, then the MVAL policy for the balanced estimator is the policy $\pi$ such that $\forall \x \in \calX, \action \in \calA: (1-\alpha)\oldpolicy(\action|x) + \alpha \policy(a|x) = \minvarpolicy(\action|\x)$.
\end{proposition}
{\sc Proof Sketch.} The basic argument is that the balanced estimator variance is the IPS variance with $\blendedpolicy$ instead of $\oldpolicy$, so if $\evalpolicy$ can make $\blendedpolicy=\minvarpolicy$ then that is optimal. Detailed proof is shown in Appendix \ref{proof:bigalpha}. \qedsymbol{}

The fact that MVAL aims to augment the existing data $\oldD$ so that the overall data looks like it was all sampled from the $\targetpolicy$ (in the case of constant $\expectedsquarer$) has an interesting implication for the overall utility. In particular, MVAL will ensure an overall utility as if all of $\oldD$ and $\evalD$ had been sampled from $\targetpolicy$. Since it is the prior belief in many A/B tests that the target policy $\targetpolicy$ is better than the logging policy $\oldpolicy$, this means that MVAL will improve utility in addition to sampling the most informative data.

\subsubsection{Introducing a New Action}

A common way a new target policy is different from the logging policy is through the introduction of a new action (e.g., a new movie). The following shows that MVAL's behavior matches the intuition that this new action should now be sampled by the augmentation policy.

\begin{proposition}[Introducing new actions]
\label{thm:newaction}
If there is an action $a \in \calA$ such that $\targetpolicy(a|x) > 0 $ but $\oldpolicy(a|x) = 0$, then the variance minimizing $\evalpolicy$ is such that $\evalpolicy(a|x) > 0$.
\end{proposition}

\begin{proof}This is simply a restatement of Theorem \ref{thm:unbiased}, which states that MVAL produces unbiased estimates for any $x$ such that $\expectedsquarer > 0$. If $\targetpolicy(a|x) > 0 = \oldpolicy(a|x)$ and $0 = \evalpolicy(a|x)$, then it would be a biased estimate. This contradicts Theorem \ref{thm:unbiased}, so $\evalpolicy(a|x) > 0$.\end{proof}

\subsubsection{Deterministic Target Policies}
\begin{proposition}[Deterministic target policy]\label{thm:deterministic} If $\targetpolicy$ is deterministic, then the optimal augmentation policy for single policy evaluation using the balanced estimator is $\targetpolicy$.\end{proposition}
{\sc Proof Sketch.} The basic argument is that if $\targetpolicy(a|x)=0$ for some action, then that action contributes nothing to the balanced estimator variance, and since the variance contribution for each action decreases with more probability, using the deterministic target policy is optimal. Detailed proof is shown in Appendix \ref{proof:deterministic}. \qedsymbol{}

\subsubsection{Variance Reduction through Augmentation Logging can be Large}

One final point is that adding even a single augmentation data point can substantially decrease the estimator's variance. While the variance reduction of a single point depends on the specific logging and target policies, there is no upper bound on the variance reduction achievable.
\begin{proposition}[Variance reductions of a single data point]
\label{thm:varunbounded}
There is no upper bound on the variance decrease attained by adding a single augmentation sample.
\end{proposition}
{\sc Proof Sketch.} Consider that when the logging policy assigns almost no probability to an action with a large probability under the target policy, then that action has an arbitrarily large variance contribution. Adding a single data point with a policy guaranteed to play that action results in an arbitrarily large variance decrease, since $\frac{1}{\epsilon} - \frac{N+1}{N\epsilon + 1}$ can be made arbitrarily large. A detailed proof is shown in Appendix \ref{proof:varunbounded}.\qedsymbol{}

\subsection{Pre-Computing MVAL Policies}

So far we have assumed that we simply solve the MVAL optimization problem for each individual context $x$ as it comes in. This is realistic in most applications, since these optimization problems are convex (see Appendix \ref{proof:convex}) and no bigger than the number of available actions. However, some applications may have latency requirements where even this modest amount of computation is not feasible. We therefore ask whether we can learn a general MVAL policy that applies to any context $\x$ ahead of time, so that this policy merely needs to be executed during deployment.

We approach the problem of learning a general MVAL policy as the following optimization problem. Given a parameterized space $\policyspace$ of policies and a sample of contexts $\{x_i\}_{i=1}^N$, find the augmentation policy $\evalpolicy \in \policyspace$ that minimizes the sum of the variances over all $N$ sample contexts.
\begin{optimizationproblem}[Pre-Computed MVAL Policy] \label{op:precompmval}
\begin{align*}
\argmin_{\policy \in \policyspace} \quad & \sum_{\x_i \in \calD}\sum_{\action \in \calA}\dfrac{\targetpolicy^2(\action|\x_i)
\Ex_r\sqaren{r^2(x_i, a)}
}{(1-\alpha)\oldpolicy(a|x_i) + \alpha \pi(a|x_i)}\\
\textrm{subject to} \quad & \sum_{a\in \calA}\pi_x(a) = 1, \\
&\pi_x(a) \geq 0   \text{ for all } a \in \calA
\end{align*}
\end{optimizationproblem}

This is akin to an empirical risk minimization objective, as an augmentation policy that minimizes the objective (here variance) on a large sample of training points can be expected to also produce good objective values on new contexts under standard conditions on the capacity of $\policyspace$. We will compare these pre-computed MVAL policies to on-the-fly computed MVAL policies in the experiments.

\section{Multi-Policy Evaluation}
\label{sec:multipol}
The previous section showed how to optimally augment an existing dataset $\oldD$ when evaluating a single target policy $\targetpolicy$. However, in practice an offline A/B test may want to evaluate a set of competing target policies $\targetpolicyspace = \{\policy_1, ... \policy_k\}$, especially when we want to learn a new policy $\policy^*$ through Empirical Risk Minimization (ERM) using the balanced estimator: 
\begin{equation}
    \policy^* = \argmax_{\policy \in \targetpolicyspace} \hat R^{\text{BAL}}_{\policy}
\end{equation}
We therefore ask the question of how to compute an MVAL policy that minimizes the maximum variance for any target policy in $\targetpolicyspace$
\begin{equation} \label{eq:minmaxvar}
    \evalpolicy(A|x)= \argmin_{\evalpolicy}\max_{\policy \in \targetpolicyspace} \text{Var}\left[\hat R^{\text{BAL}}_{\policy}(x)\right],
\end{equation}
where $R^{\text{BAL}}_{\policy}(x)$ is the estimate of the expected reward of the target policy in context $x$.

While solving Equation~\eqref{eq:minmaxvar} directly can be challenging, the following change to the optimization target provides a bound on the variance for any $\targetpolicy$ in a class of policies $\targetpolicyspace$. Optimizing this bound results in the following optimization problem, which is not more complex than single policy evaluation beyond the calculation of $\max_{\policy \in \targetpolicyspace}\policy(\action|\x)$.

\begin{theorem}[Policy Class Variance Bound]
\label{thm:policyset}
Given a class of policies $\policyspace$, then $\forall \pi \in \Pi$,
$$\text{Var}\left[\hat R^{\text{BAL}}_{\pi}\right] \leq \dfrac{1}{N}\Ex_x\sqaren{\sum_{a \in \calA}\dfrac{\maxpolicy^2(a|x)\expectedsquarer}{\blendedpolicy(a|x)}}.$$
where $\maxpolicy(a|x) = \max_{\policy \in \policyspace}\policy(\action|\x)$.
\end{theorem}


\begin{proof}
By the definition, for all $\pi \in \Pi$, $$\pi(a|x) \leq \max_{\pi' \in \Pi}\pi'(a|x) = \maxpolicy(a|x).$$
For all $a, x$, since $\sigma^2(x, a)$, $r^2(x, a)$, and $\blendedpolicy(a|x)$ are all positive,
$$0 \leq \dfrac{\expectedsquarer}{\blendedpolicy(a|x)}.$$
Therefore, for all $\pi \in \Pi$, and all $a, x$,
$$\dfrac{\pi^2(a|x)\expectedsquarer}{\blendedpolicy(a|x)} < \dfrac{\maxpolicy^2(a|x)\expectedsquarer}{\blendedpolicy(a|x)}.$$
Therefore, for all $\pi \in \Pi$,
\begin{align*}
\text{Var}\left[\hat R^{\text{BAL}}_{\pi}\right] &= \dfrac{1}{N}\Ex_x\sqaren{\sum_{a \in \calA}\dfrac{\pi^2(a|x)\expectedsquarer}{\blendedpolicy(a|x)}} \\
&\leq  \dfrac{1}{N}\Ex_x\sqaren{\sum_{a \in \calA}\dfrac{\maxpolicy^2(a|x)\expectedsquarer}{\blendedpolicy(a|x)}}
\end{align*}
\end{proof}

The structure of this bound results in the same optimization problems as the ones for single-policy optimization. In particular, the augmentation policy that minimizes the maximum value of the variance bound for any $\policy \in \policyspace$ can be found by solving the following optimization problem, where $\maxpolicy$ replaces the $\targetpolicy$ of OP\ref{op:mval}.

\begin{optimizationproblem}[MVAL for Multi-Policy Evaluation] \label{op:mvalmulti}
\begin{align*}
\argmin_{\pi_{x} \in \calR^{|\calA|}} \quad & \sum_{a \in \calA}\dfrac{\maxpolicy^2(a|x)\expectedsquarer}{(1-\alpha)\oldpolicy(a|x) + \alpha \pi_x(a)}\\
\textrm{subject to} \quad & \sum_{a\in \calA}\pi_x(a) = 1, \\
&\pi_x(a) \geq 0   \text{ for all } a \in \calA
\end{align*}
\end{optimizationproblem}

This formulation also applies to the problem of augmentation logging for learning, since many algorithms involve some notion of an ambiguity class or trust region around the current learned policy $\pi$. OP\ref{op:mvalmulti} allows one to apply MVAL whenever $\maxpolicy(a|x)$ can be efficiently computed for these policy classes $\Pi$. We can even simplify OP\ref{op:mvalmulti} for certain types of trust regions.
For example, the trust region policy class $\Pi_{\targetpolicy}=\{\pi | \forall x,a: \pi(a|x) \in [\frac{1}{\tau}\cdot\targetpolicy(a|x),\tau\cdot \targetpolicy(a|x)]\}$ can be approximated by setting $\maxpolicy \approx \targetpolicy$ for small $\tau \geq 1$.
The argument is that the solution of OP\ref{op:mvalmulti} is invariant to $\tau$, since $\maxpolicy(a|x) = \tau \cdot \targetpolicy(a|x)$ as long as $\targetpolicy(a|x)\le \frac{1}{\tau}$.

\subsection{Analysis and Discussion}

We again illustrate and discuss the behavior of MVAL, now for the case of multi-policy evaluation and learning. An instructive limiting case is the situation where there is no past data, no restrictions on the policy class $\Pi$, and no knowledge about $\rewardmeansqprewardvar$. In this case, it seems that one should sample from the uniform policy. We find this intuition agrees with MVAL for the case where $\Pi$ is all valid policies, there is no logged data, and there is an uninformative reward model with $\rewardmeansqprewardvar = c >0$.

\begin{proposition}[Uninformed Multi-Policy Evaluation]
\label{thm:uniform}
If $\text{ }\Pi$ is the space of all valid policies, there is no existing logged data, and $\rewardmeansqprewardvar = c > 0$ for all $(x, a)$, then the optimal augmentation policy is the uniform distribution.
\end{proposition}

\begin{proof}
\label{proof:uniform}
First, recall that when $\alpha=1$, the optimal augmentation policy is 

$$\dfrac{\maxpolicy(a|x)\sqrt{\rewardmeansqprewardvar}}{\sum\limits_{a\in\calA} \maxpolicy (a|x)\sqrt{\rewardmeansqprewardvar}}.$$

Since $\maxpolicy(a|x) = \max_{\pi \in \Pi} \pi(a|x)$, then if if $\Pi$ is all valid policies then $\maxpolicy(a|x) = 1$ for all $a$. Without a reward distribution model, $\rewardmeansqprewardvar=c>0$, so the optimal policy is the uniform policy:

$$\dfrac{1\sqrt{\rewardmeansqprewardvar}}{\sum\limits_{a\in\calA} 1\sqrt{\rewardmeansqprewardvar}} = \dfrac{\sqrt{c}}{\sum\limits_{a\in\calA}\sqrt{c}} = \dfrac{\sqrt{c}}{|\calA|\sqrt{c}} = \dfrac{1}{|\calA|}.$$
\end{proof}

\section{Empirical Evaluation} \label{sec:experiments}


To evaluate MVAL on a real-world contextual bandit problem, we performed experiments on the Yahoo!\ Front Page Dataset \cite{yahoodata}\footnote{\label{license}This dataset was obtained from Yahoo! Webscope at \url{http://research.yahoo.com/Academic_Relations}. \label{identifiable} The users are anonymized.}. Each context in the dataset consists of a $5$-dimensional vector $\mathbf u$ representing the user, as well as a $D\times 5$ dimensional matrix with vectors $\mathbf A_i$ for each of the $D$ articles. For convenience, we only used contexts with exactly $19$ articles. The dataset includes which article was recommended (i.e., the action) and if the user clicked on the article (i.e., the reward).

When collecting this dataset, the article recommendations were chosen uniformly at random from the articles available for the context. This allows for an unbiased simulation of running a different article recommendation policy $\pi$ using rejection sampling \cite{krause_yahoo}. To construct a sample for a new policy $\pi$, we iterate through the contexts $x_i$ and sample $a' \sim \pi(\cdot|x)$ according to $\pi$. If the $a'$ sampled from our policy agrees with the observed action $a_i$, we include that $(x_i, a_i, r_i)$ tuple. This gives us an unbiased sample that comes from the same distribution as if we were running our new policy $\pi$ on the operational system.


\vspace*{-1mm}\paragraph{Model Architecture and Training.}
The policy architecture for precomputed MVAL is a feedforward neural network, ending in a softmax layer where the logit for article $j$ is based on the concatenation of the user vector to the article vectors $\mathbf u \circ \mathbf A_j$ passed through 2 fully connected ReLU \cite{relu} layers with $256, 256$ nodes before $1$ fully connected linear node. Adam is used for optimization with the standard parameters $\alpha=0.001$, $\beta_1=0.9$, $\beta_2=0.999$ \cite{adam} and a batch size of $10000$. All experiments can be run on a desktop with an RTX2080. \label{specs}

\vspace*{-1mm}\paragraph{Estimating $\rewardmeansqprewardvar$.} The first set of experiments use a uniform $\rewardmeansqprewardvar$, since these experiments explore how any decrease in variance is a result of the algorithmic improvement, rather than a result of a model of $\rewardmeansqprewardvar$. For the sequential learning experiment, we use a feedforward neural network to approximate $\Ex\sqaren{r^2(x,a)}$, trained using mean squared error to perform regression on the quantity $r^2(x,a)$ using the same architecture as the policy networks described above.

\vspace*{-1mm}\paragraph{Generating Logging and Target Policies.}\label{crosspolicy} 
The policy evaluation experiments use randomly generated target and logging policies.
We control the generating process to vary how deterministic the logging policy's actions are, and how different the target and logging policies are.
In particular, to generate the logging policy, we randomly sample a vector $\mathbf v \in \mathbb{R}^{25}$ such that $v_i \sim \mathcal{N}(\mu=0, \sigma=1)$. This vector is then multiplied against a feature vector of the cross terms between $u$ and each $A_i$ for the given context $i$ to allow for interactions between user and article features while remaining a simple policy class. Then, the articles are ranked according to this value to define the selection probabilities of the policy based on the rank. In particular, the probability of choosing each article is proportional to $\eta^{\text{rank}_i}$, where $\text{rank}_i$ is the rank of the $i$th article. This allows us to increase the determinism of the policy by increasing $\eta$. We use the same construction to generate a target policy, but shift a $\delta$ fraction of the probability weight from the top-ranked article under the logging policy to the article that is ranked second under the logging policy. Increasing $\delta$ allows us to increase the difference between the target and logging policy.

\begin{figure}[t]
\includegraphics[width=.47\textwidth]{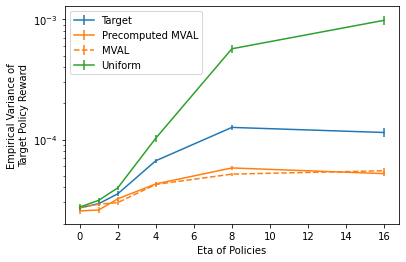}
\vspace*{-4mm}
\caption{Variance of the balanced estimator while holding $\delta=0.4$ and increasing values of $\eta$, which increases the determinism of the policies. 
Error bars are the standard error based on $20$ trials. 
Note the logarithmic y axis.
}
\label{fig:eta}
\end{figure}

\begin{figure}[t]
\includegraphics[width=.47\textwidth]{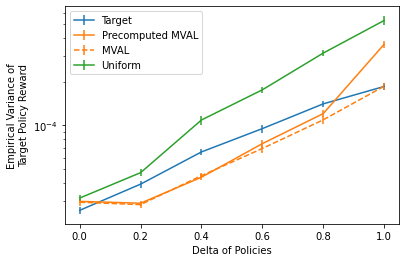}
\vspace*{-4mm}
\caption{Variance of the balanced estimator while holding $\eta=4$ and increasing values of $\delta$, which increases the difference between the logging and target policies.
Error bars are based on $20$ trials.
}
\label{fig:frac}
\end{figure}

\subsection{Single Policy Evaluation}
The first experiments explore the effectiveness of MVAL for evaluating a single target policy while varying $\eta$ and $\delta$ for the logging and target policy generation. The reported variance is the empirical variance of $50$ sampled value estimates for the target policy. Each value estimate is generated by first sampling $900$ data points from the logging policy. We then compare different augmentation logging policies that are allow to sample $100$ additional data point. Specifically, we compare MVAL and precomputed MVAL to using the target policy or the uniform policy for augmentation logging. Finally, we estimating the value of the target policy using the balanced estimator on all $1000$ data points. The standard error bars are based on of $20$ such runs for each parameter setting.

In Figure \ref{fig:eta}, we see that using MVAL substantially outperforms using the target policy or uniform policy for augmentation logging for a range determinism factors $\eta$. The difference $\delta$ between logging and target policy is fixed at $0.4$ in this experiment. At $\eta=0$, the logging policy is the uniform distribution, and all policies have similar performance. We find that precomputed MVAL performs comparable to MVAL over the whole range of $\eta$.

In Figure \ref{fig:frac}, we vary the value of $\delta$ while the determinism factor $\eta$ is fixed. Again, we find that using MVAL to generate the augmentation policy substantially outperforms using the target policy or uniform distribution for a broad range of settings. Increasing the change fraction $\delta$ decreases the performance of all methods as expected. At $\delta=1$, MVAL and target tie because the target policy is fairly deterministic, and per Proposition \ref{thm:deterministic} the target policy is the variance-optimal augmentation policy. At  $\delta=1$, precomputed MVAL performs worse than exact MVAL, possibly because it is training a policy based on $(x, a, r)$ tuples close to the logging policy rather than the quite different target.

\begin{figure}
\includegraphics[width=.47\textwidth]{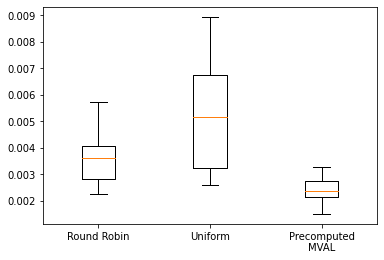}
\vspace*{-3mm}
\caption{Variance of multiple policy evaluation for 3 policies with determinism factor $\eta=4$ and change fraction $\delta=0.4$. Each box plot contains all the empirical variances of all 3 policy estimates, for $20$ trials of the full experiment.
}
\label{fig:kway}
\end{figure}

\subsection{Multiple Policy Evaluation}

The next experiment explores the effectiveness of MVAL for multi-policy evaluation. The policies for this experiment were generated as before with determinism factor $\eta=4$ and a difference between logging and target policy of $\delta=0.4$. However, the $3$ target policies were constructed by shifting $\delta$ fraction of the logging policy's top action's probability to the $2$nd, $3$rd, and $4$th actions. The experiment variance is the empirical variance of $100$ runs of using the balanced estimator to estimate the target policy reward. The round-robin strategy uses each target policy to sample $333$ augmentation data points, while the other strategies collected $999$ augmentation data points from the uniform or precomputed MVAL policy. The original logging policy provided $9001$ data points for a total of $10000$ contexts. 
As shown in Figure \ref{fig:kway}, precomputed MVAL substantially outperforms the round-robin and uniform strategies.

\subsection{Policy Learning}
\begin{figure}
\includegraphics[width=.47\textwidth]{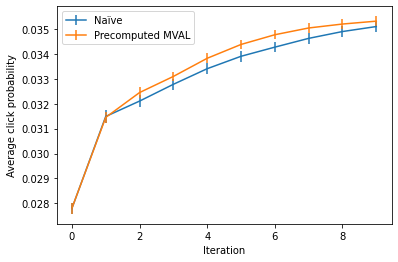}
\vspace*{-3mm}
\caption{Average cumulative reward for sequentially using naïve or MVAL augmentation logging at each iteration for $\tau=1$. Error bars are the standard error based on $50$ trials.}
\label{fig:sbo}
\end{figure}

The final experiment explores the effectivess of MVAL for policy learning over time, where we repeatedly gather $10000$ additional augmentation data points and retrain the model using POEM \cite{poem} with the balanced estimator as policy learner.
In the initial iteration we gather $10000$ data points from the uniform policy, following the optimal MVAL strategy in this situation according to Proposition \ref{thm:uniform}. After that, each iteration $t$ trains a policy $\pi^t$ using POEM \cite{poem} with the balanced estimator as policy learner. We compare using MVAL for augmentation logging to naively logging additional data from the target policy $\pi^t$ in each iteration. We use precomputed MVAL with $\rewardmeansqprewardvar = m^t(x, a)$, where $m^t$ approximates $\expectedsquarer$ using regression. Furthermore, we approximate $\maxpolicy \approx \pi^t$ as discussed in Section~\ref{sec:multipol}. For both the naive augmentation logging using $\pi^t$ and MVAL augmentation logging we train POEM for $1000$ epochs, which is more than sufficient for convergence. The POEM clipping parameter is set to $10000$, which is chosen to maximize the performance of the naive method. Figure \ref{fig:sbo} shows that augmentation logging via MVAL outperforms naively using the target policy for augmentation logging.

\section{Conclusions}

We introduced and formalized the problem of augmentation logging, and derived MVAL as a principled and practical method for variance-optimal data gathering for off-policy evaluation. We extended the approach to multi-policy evaluation and batch learning, and find that it can substantially improve estimation and learning quality over naive methods. This work opens up a number of directions for future work. For example, contextual-bandit problems with combinatorial actions like slates \cite{slate} raise additional challenges, but they also provide structure that connects the observations for different actions. It is interesting to explore how this structure can inform augmentation logging for improved bias/variance trade-offs.



\section{Broader Impacts and Ethics}
The main impact of this work is explaining how to decrease the variance for off-policy estimators for contextual bandits by collecting additional augmentation data. This is most likely to be used in building better recommender systems. Our particular method makes it easier to collect data in order evaluate particular policy classes, even when they differ from the logged policies. This will likely make it easier and lower impact for companies to evaluate bigger changes to their recommender systems. Depending on the changes, this could be good or bad. \label{negative}
While this paper uses article clicks as the reward in our experiments, we caution against viewing this as the sole metric of recommender quality, and encourage others to investigate more holistic methods.
\label{consent}The original paper for the Yahoo! Frontpage Dataset did not say that users opted into the experiment. However, we feel that users were unlikely to be harmed by inclusion in the experiment, which randomized recommendations amongst human-curated articles that were available at the time, and revealed no personally identifiable information.

\section*{Acknowledgments}
This research was supported in part by NSF Awards IIS1901168 and IIS-2008139. All content represents the opinion of the authors, which is not necessarily shared or endorsed by their respective employers and/or sponsors.


\bibliography{bibliography}
\bibliographystyle{icml2022}

\newpage
\appendix
\onecolumn
\section{Proof Appendix}

\subsection{Proof of Equation \ref{equation:balancedvar}, Variance of the Balanced Estimator}
We want to prove Equation \ref{equation:balancedvar}, which gives the variance of the Balanced Estimator
$$
\text{Var}\left[\hat R^{\text{BAL}}_{\targetpolicy}\right]
= \dfrac{1}{N}\paren{\mathbb{E}_{x}\left[\sum_{a \in \calA}
\dfrac{\targetpolicy^2(a|x)\expectedsquarer}{\blendedpolicy(a|x)}
\right] -
R_{\targetpolicy}^2}.$$

In order to do so, we will first introduce a lemma gives the variance of a single IPS term, then describe the variance of the IPS estimator, then move on to the balanced estimator.

\subsubsection{Variance of a single IPS term}
To simplify our proofs, we introduce a lemma.
\begin{lemma}[Variance of a single IPS term]
\label{lemma:singleterm}
For any policy $\pi$ such that the support of $\pi$ is a superset of the support of $\targetpolicy$,
$$\Var\sqaren{\dfrac{\targetpolicy(a_i|x_i)}{\pi(a_i|x_i)}r_i} = \Ex_{x}\sqaren{\sum_{a\in\calA}\dfrac{\targetpolicy^2(a|x)}{\pi(a|x)}\paren{\bar r^2(x, a) + \sigma^2(x, a)}} - R_{\targetpolicy}^2.$$
\end{lemma}
\begin{proof}
Since $\Var[X] = \Ex[X^2] - \Ex[X]^2$,
$$
\Var\sqaren{\dfrac{\targetpolicy(a_i|x_i)}{\pi(a_i|x_i)}r_i} = \Ex\sqaren{\paren{\dfrac{\targetpolicy(a_i|x_i)}{\pi(a_i|x_i)}r_i}^2} - \Ex\sqaren{\dfrac{\targetpolicy(a_i|x_i)}{\pi(a_i|x_i)}r_i}^2.
$$
Since the IPS estimator is unbiased if the support of the logging policy $\pi(a|x)$ is a superset of $\targetpolicy(a|x)$ for all $x$ in the support of $\Pr(x)$, if we denote $\Ex_{\targetpolicy}[r_i] = R_{\targetpolicy}$, then
$$
\Var\sqaren{\dfrac{\targetpolicy(a_i|x_i)}{\pi(a_i|x_i)}r_i} = \Ex\sqaren{\paren{\dfrac{\targetpolicy(a_i|x_i)}{\pi(a_i|x_i)}r_i}^2} - R_{\targetpolicy}^2.
$$
Breaking up the expectations, we have the following:
\setcounter{equation}{0}
\begin{align}
\Ex\sqaren{\paren{\dfrac{\targetpolicy(a_i|x_i)}{\pi(a_i|x_i)}r_i}^2} &= \Ex_{x\sim\Pr(x)}\Ex_{a\sim\pi(a|x)}\Ex_{r_i\sim r(x, a)}\sqaren{\paren{\dfrac{\targetpolicy(a|x)}{\pi(a|x)}r_i}^2}\\
&= \Ex_{x\sim\Pr(x)}\Ex_{a\sim\pi(a|x)}\sqaren{\dfrac{\targetpolicy^2(a|x)}{\pi^2(a|x)}\Ex_{r_i\sim r(x, a)}\sqaren{r_i^2}}\\
&= \Ex_{x\sim\Pr(x)}\sqaren{\sum_{a\in\calA}\dfrac{\targetpolicy^2(a|x)}{\pi^{\bcancel{2}}(a|x)}\Ex_{r_i\sim r(x, a)}\sqaren{r_i^2}\bcancel{\pi(a|x)}}\\
&= \Ex_{x\sim\Pr(x)}\sqaren{\sum_{a\in\calA}\dfrac{\targetpolicy^2(a|x)}{\pi(a|x)}\Ex_{r_i\sim r(x, a)}\sqaren{r_i^2}}\\
&= \Ex_{x\sim\Pr(x)}\sqaren{\sum_{a\in\calA}\dfrac{\targetpolicy^2(a|x)}{\pi(a|x)}\paren{\bar r^2(x, a) + \sigma^2(x, a)}}\\
&= \Ex_{x}\sqaren{\sum_{a\in\calA}\dfrac{\targetpolicy^2(a|x)}{\pi(a|x)}\paren{\bar r^2(x, a) + \sigma^2(x, a)}}
\end{align}
Equation 1 comes from iterated expectation, 2 from the fact that $\Ex_{r_i\sim r(x, a))}$ is conditioned on $x$ and $a$ (so they can be factored out), 3 from the definition of $\Ex_{a\sim \pi(a|x)}$, 4 from cancelling out terms, 5 from the fact that $\Var[x] = \Ex[x^2] - \Ex[x]^2$, and 6 is simplifying the notation.

Substituting this back into the earlier variance term, we have 

\begin{align*}
\Var\sqaren{\dfrac{\targetpolicy(a_i|x_i)}{\pi(a_i|x_i)}r_i} &= \Ex_{x}\sqaren{\sum_{a\in\calA}\dfrac{\targetpolicy^2(a|x)}{\pi(a|x)}\paren{\bar r^2(x, a) + \sigma^2(x, a)}} - R_{\targetpolicy}^2\\
&= \Ex_{x}\sqaren{\sum_{a\in\calA}\dfrac{\targetpolicy^2(a|x)}{\pi(a|x)}\paren{\expectedsquarer}} - R_{\targetpolicy}^2,
\end{align*}

 as desired.
\end{proof}

\subsubsection{Variance of the Balanced Estimator}

Now we are ready to show that the variance of the balanced estimator is $$
\text{Var}\left[\hat R^{\text{BAL}}_{\targetpolicy}\right]
= \dfrac{1}{N}\paren{\mathbb{E}_{x}\left[\sum_{a \in \calA}
\dfrac{\targetpolicy^2(a|x)\expectedsquarer}{\blendedpolicy(a|x)}
\right] -
R_{\targetpolicy}^2}.$$

\begin{proof}
\label{proof:balancedvar}
Following the first steps of of the proof for equation \ref{equation:ipsvar}, we have the following variance for the balanced estimator, with the first variance taken over $\oldpolicy$ and the second taken over $\evalpolicy$:
$$\Var\sqaren{\hat R_{\targetpolicy}^{\BAL}} = \dfrac{1}{N^2}\paren{\oldn\Var_{\oldpolicy}\sqaren{ \dfrac{\targetpolicy(a_i|x_i)}{\blendedpolicy(a_i|x_i)}r_i} + \evaln\Var_{\evalpolicy}\sqaren{ \dfrac{\targetpolicy(a_i'|x_i')}{\blendedpolicy(a_i'|x_i')}r_i'}}.$$

Following the algebra for Lemma \ref{lemma:singleterm}, but stopping at equation 3, we have the following expression for a single balanced estimator variance term taken over a policy $\pi$.
$$\Var\sqaren{\dfrac{\targetpolicy(a_i|x_i)}{\blendedpolicy(a_i|x_i)}r_i} 
= \Ex_{x}\sqaren{\sum_{a\in\calA}\dfrac{\targetpolicy^2(a|x)}{\blendedpolicy^{2}(a|x)}\expectedsquarer\pi(a|x)} - R_{\targetpolicy}^2.$$

Combining these two facts, we have the following:
\begin{align*}
\Var\sqaren{\hat R_{\targetpolicy}^{\BAL}} &=
\dfrac{\oldn}{N^2}\paren{\Ex_{x}\sqaren{\sum_{a\in\calA}\dfrac{\targetpolicy^2(a|x)}{\blendedpolicy^{2}(a|x)}\expectedsquarer\oldpolicy(a|x)} - R_{\targetpolicy}^2}\\
&\hspace{11pt}+\dfrac{\evaln}{N^2}\paren{\Ex_{x}\sqaren{\sum_{a\in\calA}\dfrac{\targetpolicy^2(a|x)}{\blendedpolicy^{2}(a|x)}\expectedsquarer\evalpolicy(a|x)} - R_{\targetpolicy}^2}\\
&=\dfrac{1}{N}\paren{\Ex_{x}\sqaren{\sum_{a\in\calA}\dfrac{\targetpolicy^2(a|x)}{\blendedpolicy^{2}(a|x)}\expectedsquarer(1-\alpha)\oldpolicy(a|x)}}\\
&\hspace{11pt}+\dfrac{1}{N}\paren{\Ex_{x}\sqaren{\sum_{a\in\calA}\dfrac{\targetpolicy^2(a|x)}{\blendedpolicy^{2}(a|x)}\expectedsquarer\alpha\evalpolicy(a|x)}}\\
&\hspace{11pt}-\dfrac{1}{N}R_{\targetpolicy}^2\\
&=\dfrac{1}{N}\Ex_{x}\sqaren{\sum_{a\in\calA}\dfrac{\targetpolicy^2(a|x)\expectedsquarer}{\blendedpolicy^{2}(a|x)}\paren{(1-\alpha)\oldpolicy(a|x) + \alpha\evalpolicy(a|x)}}\\
&\hspace{11pt}-\dfrac{1}{N}R_{\targetpolicy}^2\\
&=\dfrac{1}{N}\paren{\Ex_{x}\sqaren{\sum_{a\in\calA}\dfrac{\targetpolicy^2(a|x)}{\blendedpolicy^{\bcancel{2}}(a|x)}\expectedsquarer\bcancel{\blendedpolicy(a|x)}}-R_{\targetpolicy}^2},
\end{align*}
as desired.

In the same way that the $R_{\targetpolicy}^2/N$ term comes from the fact that $N = \oldn + \evaln$, we can leverage the fact that $\alpha=\evaln/(\oldn + \evaln)$ and $\blendedpolicy = \alpha\evalpolicy + (1-\alpha)\oldpolicy$ to combine the variance expressions from each policy.
\end{proof}

\subsection{IPS Estimator is unsuitable for problem setting}
\label{proof:ipsnogo}
\paragraph{IPS estimator} The inverse propensity score estimator is one of the oldest off-policy evaluation techniques. Assuming that the past policy $\oldpolicy$ is known, simply weigh each data point by the inverse probability of the action being taken by the policy \cite{ips}. This is simple to express in the logged bandit feedback setting. If all the old data $\calD = \{(x_i, a_i, r_i)\}_{i=1}^{N}$ is sampled from $\oldpolicy$, then the estimator is as follows.

$$\hat R^{\IPS}_{\targetpolicy} = \dfrac{1}{N}\sum_{i=1}^N \dfrac{\targetpolicy(a_i|x_i)}{\oldpolicy(a_i|x_i)} r_i$$

\paragraph{IPS estimator in data augmentation setting}
It is straightforward to adapt this to the logged bandit feedback with data augmentation setting.  Consider the case where all of the old data $\calD_{\old} = \{(x_i, a_i, r_i)\}_{i=1}^{\oldn}$ is sampled from a single policy $\oldpolicy$, and the new data $\calD_{\eval} = \{(x_i', a_i', r_i')\}_{i=1}^{\evaln}$ is sampled from $\evalpolicy$, then the IPS Estimator is as follows, where $N = \oldn + \evaln$.

\begin{equation}
\label{equation:ips}
\hat R_{\targetpolicy}^{\IPS}
= \dfrac{1}{N}\paren{\sum_{i=1}^{\oldn} \dfrac{\targetpolicy(a_i|x_i)}{\oldpolicy(a_i|x_i)}r_i + \sum_{i=1}^{\evaln} \dfrac{\targetpolicy(a_i'|x_i')}{\evalpolicy(a_i'|x_i')}r_i'}
\end{equation}

This has the following variance as proven in Appendix \ref{proof:ipsvar}.

\begin{align*}
\Var\sqaren{\hat R_{\targetpolicy}^{\IPS}} &= \dfrac{\oldn}{N^2}\Ex_{x}\sqaren{\sum_{a\in\calA}\dfrac{\targetpolicy^2(a|x)\expectedsquarer}{\oldpolicy(a|x)}}+\dfrac{\evaln}{N^2}\Ex_{x}\sqaren{\sum_{a\in\calA}\dfrac{\targetpolicy^2(a|x)\expectedsquarer}{\evalpolicy(a|x)}}-\dfrac{1}{N^2}R_{\targetpolicy}^2
\end{align*}

\paragraph{Minimum variance $\evalpolicy$ for IPS is independent of $\oldpolicy$} In this expression, it is clear that the partial derivatives of the variance with respect to $\evalpolicy(a|x)$ depend only on $\targetpolicy$, the term $\expectedsquarer$, and $\Pr(x)$, but not the historical policy $\oldpolicy$ or any data points logged under the old policy. This means that the minimal variance estimator policy would be the same variance minimizing policy for if $\evalpolicy$ were the only policy observed, and cannot account for anything related to the previous policy $\oldpolicy$.

\subsubsection{Variance of the IPS Estimator}
Now, we show that the variance of the IPS Estimator in the batch learning with bandit feedback and data augmentation setting is as follows. 

\begin{equation}\label{equation:ipsvar}
\Var\sqaren{\hat R_{\targetpolicy}^{\IPS}} = \dfrac{\oldn}{N^2}\Ex_{x}\sqaren{\sum_{a\in\calA}\dfrac{\targetpolicy^2(a|x)\expectedsquarer}{\oldpolicy(a|x)}} +\dfrac{\evaln}{N^2}\Ex_{x}\sqaren{\sum_{a\in\calA}\dfrac{\targetpolicy^2(a|x)\expectedsquarer}{\evalpolicy(a|x)}} -\dfrac{1}{N}R_{\targetpolicy}^2
\end{equation}

\begin{proof}
\label{proof:ipsvar}
Recall the IPS estimator in the data augmentation setting.
\begin{align*}
\hat R_{\targetpolicy}^{\IPS}
&= \dfrac{1}{N}\paren{\sum_{i=1}^{\oldn} \dfrac{\targetpolicy(a_i|x_i)}{\oldpolicy(a_i|x_i)}r_i + \sum_{i=1}^{\evaln} \dfrac{\targetpolicy(a_i'|x_i')}{\evalpolicy(a_i'|x_i')}r_i'}\\
\Var\sqaren{\hat R_{\targetpolicy}^{\IPS}} &= \Var\sqaren{\dfrac{1}{N}\paren{\sum_{i=1}^{\oldn} \dfrac{\targetpolicy(a_i|x_i)}{\oldpolicy(a_i|x_i)}r_i + \sum_{i=1}^{\evaln} \dfrac{\targetpolicy(a_i'|x_i')}{\evalpolicy(a_i'|x_i')}r_i'}}
\end{align*}

Since $\Var\sqaren{aX} = a^2\Var\sqaren{X}$,

$$\Var\sqaren{\hat R_{\targetpolicy}^{\IPS}} = \dfrac{1}{N^2}\Var\sqaren{\sum_{i=1}^{\oldn} \dfrac{\targetpolicy(a_i|x_i)}{\oldpolicy(a_i|x_i)}r_i + \sum_{i=1}^{\evaln} \dfrac{\targetpolicy(a_i'|x_i')}{\evalpolicy(a_i'|x_i')}r_i'}.$$

Since each term of the sum is independent conditioned on the target and logging policies $\targetpolicy$ and $ \oldpolicy$, and because $\Var[X+Y] = \Var[X] + \Var[Y]$ for independent variables,

$$\Var\sqaren{\hat R_{\targetpolicy}^{\IPS}} = \dfrac{1}{N^2}\paren{\oldn\Var\sqaren{ \dfrac{\targetpolicy(a_i|x_i)}{\oldpolicy(a_i|x_i)}r_i} + \evaln\Var\sqaren{ \dfrac{\targetpolicy(a_i'|x_i')}{\evalpolicy(a_i'|x_i')}r_i'}}.$$

Now using Lemma \ref{lemma:singleterm} for each variance term, we have the following:
\setcounter{equation}{0}
\begin{align}
\Var\sqaren{\hat R_{\targetpolicy}^{\IPS}} &= \dfrac{1}{N^2}\paren{\oldn\Var\sqaren{ \dfrac{\targetpolicy(a_i|x_i)}{\oldpolicy(a_i|x_i)}r_i} + \evaln\Var\sqaren{ \dfrac{\targetpolicy(a_i'|x_i')}{\evalpolicy(a_i'|x_i')}r_i'}}\\
&= \dfrac{\oldn}{N^2}\paren{\Ex_{x}\sqaren{\sum_{a\in\calA}\dfrac{\targetpolicy^2(a|x)\expectedsquarer}{\oldpolicy(a|x)}} -R_{\targetpolicy}^2} + \dfrac{\evaln}{N^2}\paren{\Ex_{x}\sqaren{\sum_{a\in\calA}\dfrac{\targetpolicy^2(a|x)\expectedsquarer}{\evalpolicy(a|x)}}-
R_{\targetpolicy}^2}\\
&= \dfrac{\oldn}{N^2}\Ex_{x}\sqaren{\sum_{a\in\calA}\dfrac{\targetpolicy^2(a|x)\expectedsquarer}{\oldpolicy(a|x)}}+\dfrac{\evaln}{N^2}\Ex_{x}\sqaren{\sum_{a\in\calA}\dfrac{\targetpolicy^2(a|x)\expectedsquarer}{\evalpolicy(a|x)}}-\dfrac{\oldn+\evaln}{N^2}R_{\targetpolicy}^2\\
&= \dfrac{\oldn}{N^2}\Ex_{x}\sqaren{\sum_{a\in\calA}\dfrac{\targetpolicy^2(a|x)\expectedsquarer}{\oldpolicy(a|x)}}+\dfrac{\evaln}{N^2}\Ex_{x}\sqaren{\sum_{a\in\calA}\dfrac{\targetpolicy^2(a|x)\expectedsquarer}{\evalpolicy(a|x)}}-\dfrac{1}{N}R_{\targetpolicy}^2.
\end{align}
Line 2 comes from using Lemma \ref{lemma:singleterm}, then line 3 is linearity, then line 4 is since $N=\oldn +\evaln$, and we have the desired equation.
\end{proof}


\subsection{Proof of Equation \ref{equation:ipsminvar}, Minimum Variance policy for IPS}
This is a proof for equation \ref{equation:ipsminvar}, which states that the MVAL policy for the IPS estimator is as follows:

$$\evalpolicy^{\BAL}(a|x) = \evalpolicy^{\IPS}(a|x)  = \dfrac{\targetpolicy (a|x)\sqrt{\rewardmeansqprewardvar}}{\sum\limits_{a\in\calA} \targetpolicy (a|x)\sqrt{\rewardmeansqprewardvar}} :=  \minvarpolicy(a|x).$$

\begin{proof}
\label{proof:minvar}
Recall the optimization problem OP\ref{op:mval} for the optimal augmentation policy for $\pi(\cdot|x)$.
\begin{align*}
\argmin_{\pi_{x} \in \calR^{|\calA|}} \quad & \sum_{a \in \calA}\dfrac{\maxpolicy^2(a|x)\parenrewardmeansqprewardvar}{(1-\alpha)\oldpolicy(a|x) + \alpha \pi_x(a)}\\
\textrm{subject to} \quad & \sum_{a\in \calA}\pi_x(a) = 1, \\
&\pi_x(a) \geq 0   \text{ for all } a \in \calA
\end{align*}

If $\alpha = 1$, the variance term becomes 
$$\sum_{a \in \calA}\dfrac{\maxpolicy^2(a|x)\parenrewardmeansqprewardvar}{\pi_x(a)}.$$

Incorporating the constraints using Lagrange multipliers, we have the following objective:

$$L(\pi_x, \lambda) = \sum_{a \in \calA}\dfrac{\maxpolicy^2(a|x)\parenrewardmeansqprewardvar}{\pi_x(a)} - \lambda\paren{1 - \sum_{a\in \calA}\pi_x(a)}$$

The partial derivatives of this objective with respect to $\pi_x(a)$ are
\begin{equation}
    \label{equation:ipspartials}
    \dfrac{\partial}{\partial \pi_x(a)}L(\pi_x, \lambda) = -\alpha
\dfrac{\targetpolicy^2(a|x)\parenrewardmeansqprewardvar}{\pi_x^2(a)} + \lambda 
\end{equation}

Setting this equal to $0$, we have the following for each $\pi_x(a)$.
\begin{align*}
0 &= -\dfrac{\targetpolicy^2(a|x)\parenrewardmeansqprewardvar}{\pi_x^2(a)} + \lambda\\
\lambda &= \dfrac{\targetpolicy^2(a|x)\parenrewardmeansqprewardvar}{\pi_x^2(a)}\\
\pi_x^2(a)  &= \dfrac{\targetpolicy^2(a|x)\parenrewardmeansqprewardvar}{\lambda}  \\
\pi_x(a)  &= \dfrac{\targetpolicy(a|x)\sqrt{\rewardmeansqprewardvar}}{\sqrt{\lambda}}
\end{align*}

The partial derivative with respect to $\lambda$ is 
$$\dfrac{\partial}{\partial \lambda}L(\pi_x, \lambda) = (1 - \sum_{a\in \calA}\pi_x(a)).
$$

Setting this equal to $0$, we have

\begin{align*}
0 &= 1 - \sum_{a\in \calA}\pi_x(a)\\
1 &= \sum_{a\in \calA}\pi_x(a)\\
1 &= \sum_{a\in \calA}\dfrac{\targetpolicy(a|x)\sqrt{\rewardmeansqprewardvar}}{\sqrt{\lambda}} \\
\sqrt{\lambda} &= \sum_{a\in \calA}\targetpolicy(a|x)\sqrt{\rewardmeansqprewardvar}
\end{align*}

Therefore, the optimal solution for $\pi_x(a)$ is

$$\dfrac{\targetpolicy (a|x)\sqrt{\rewardmeansqprewardvar}}{\sum\limits_{a\in\calA} \targetpolicy (a|x)\sqrt{\rewardmeansqprewardvar}},$$

as desired.
\end{proof}

\subsection{Proof of Proposition \ref{thm:bigalpha}, Large-$\alpha$ Closed-Form Solution for MVAL}

Recall Proposition \ref{thm:bigalpha}, which states that if $\exists \policy: \forall \x \in \calX, \action \in \calA: (1-\alpha)\oldpolicy(\action|x) + \alpha \policy(a|x) = \minvarpolicy(\action|\x)$, then the MVAL policy for the balanced estimator is $\policy$.

\begin{proof}
\label{proof:bigalpha}
Consider the variance of the IPS estimator with augmentation policy $\pi$.
$$\text{Var}\left[\hat R^{\text{BAL}}_{\targetpolicy}\right]
= \mathbb{E}_{x}\left[\sum_{a \in \calA}
\dfrac{\targetpolicy^2(a|x)\expectedsquarer}{\pi(a|x)}
\right] -
R_{\targetpolicy}^2.$$

By Equation~\eqref{equation:ipsminvar}, the augmentation policy which minimizes this variance for a given $x$ is

$$
\pi^*(a|x) = \minvarpolicy(a|x) = \dfrac{\targetpolicy (a|x)\sqrt{ r^2(x, a) + \sigma^2(x, a)}}{\sum\limits_{a\in\calA} \targetpolicy (a|x)\sqrt{\bar r^2(x, a) + \sigma^2(x, a)}}.
$$

The variance for $\minvarpolicy$ is a lower bound on the variance for any valid policy $\pi$.

$$
\sum_{a \in \calA}
\dfrac{\targetpolicy^2(a|x)\expectedsquarer}{\minvarpolicy(a|x)}
 -
R_{\targetpolicy}^2 \leq 
\sum_{a \in \calA}
\dfrac{\targetpolicy^2(a|x)\expectedsquarer}{\pi(a|x)}
 -
R_{\targetpolicy}^2.
$$

Since this is a lower bound for any policy $\pi$, t is a lower bound on the variance for any blended policy $(1-\alpha)\oldpolicy(a|x) + \alpha\pi_x(a)$.

$$
\sum_{a \in \calA}
\dfrac{\targetpolicy^2(a|x)\expectedsquarer}{\minvarpolicy(a|x)}
 -
R_{\targetpolicy}^2 \leq 
\sum_{a \in \calA}
\dfrac{\targetpolicy^2(a|x)\expectedsquarer}{(1-\alpha)\oldpolicy(a|x) + \alpha\pi_x(a)}
 -
R_{\targetpolicy}^2\\
$$

Therefore, if $\exists \pi_x: \forall a \in \calA: (1-\alpha)\oldpolicy(a|x) + \alpha\pi_x = \minvarpolicy(a|x)$, then there is no alternative $\pi_x' \neq \pi_x$  that obtains a lower variance when used as $\evalpolicy$ than $\pi_x$ .
\end{proof}

\subsection{Proof of convexity}
\begin{proof}[Proof of convexity]
\label{proof:convex}
Recall Equation \ref{equation:balancedvar}, which states that the variance of the balanced estimator is 
$$
\text{Var}\left[\hat R^{\text{BAL}}_{\targetpolicy}\right]
= \dfrac{1}{N}\paren{\mathbb{E}_{x}\left[\sum_{a \in \calA}
\dfrac{\targetpolicy^2(a|x)\expectedsquarer}{\blendedpolicy(a|x)}
\right] -
R_{\targetpolicy}^2}.$$

The partial derivative of the variance of the balanced estimator with respect to $\evalpolicy(a|x)$ is $$\dfrac{\partial}{\partial \evalpolicy(a|x)}\text{Var}\left[\hat R^{\text{BAL}}_{\targetpolicy}\right] = -\dfrac{\alpha}{N}
\dfrac{\targetpolicy^2(a|x)\expectedsquarer}{\paren{(1-\alpha)\oldpolicy(a|x) + \alpha\evalpolicy(a|x)}^2}
\Pr(x).$$
This partial derivative is always negative, because there is a negative sign, and $\pi(a|x) > 0, \pi^2(a|x) > 0$ for all valid $\pi$, and $\rewardmeansq$ and $\rewardvar$ are also always positive. 

Now consider its second derivatives.

For $x', a' \neq x, a$,

$$\dfrac{\partial^2}{\partial \evalpolicy(a|x)\partial \evalpolicy(a'|x')}\text{Var}\left[\hat R^{\text{BAL}}_{\targetpolicy}\right] = 0.$$

Because these are all zero, the Hessian matrix of the variance of the balanced estimator is diagonal.

For $x, a$,

$$\dfrac{\partial^2}{\partial^2 \evalpolicy(a|x)}\text{Var}\left[\hat R^{\text{BAL}}_{\targetpolicy}\right] = \dfrac{2\alpha^2}{N}
\dfrac{\targetpolicy^2(a|x)\expectedsquarer}{\paren{(1-\alpha)\oldpolicy(a|x) + \alpha\evalpolicy(a|x)}^3}
\Pr(x).$$

Note that every term of this equation is always positive, so this term is always positive.

This means that every term in the diagonal of the Hessian is positive, and since the Hessian is a diagonal matrix, this means that the Hessian is positive-definite, and therefore the optimization problem is convex. \end{proof}

\subsection{Proof of Proposition \ref{thm:deterministic}, Target Policy Variance Optimality for Deterministic Target Policies}
\begin{proof}
\label{proof:deterministic}
Recall the variance expression for the balanced estimator from equation \ref{equation:balancedvar}:
$$\text{Var}\left[\hat R^{\text{BAL}}_{\targetpolicy}\right]
= \dfrac{1}{N}\paren{\mathbb{E}_{x}\left[\sum_{a \in \calA}
\dfrac{\targetpolicy^2(a|x)\expectedsquarer}{(1-\alpha)\oldpolicy(a|x) + \alpha\evalpolicy(a|x)}
\right] -
R_{\targetpolicy}^2}.$$

If $\targetpolicy$ is deterministic, then for all $x\in\calX$, there is only one $a\in\calA$ such that $\targetpolicy(a|x) > 0$. Call that action $a_x$, and note $\targetpolicy(a_x|x) = 1$. This simplifies the variance expression to:
$$\text{Var}\left[\hat R^{\text{BAL}}_{\targetpolicy}\right]
= \dfrac{1}{N}\paren{\mathbb{E}_{x}\left[
\dfrac{1^2\Ex\sqaren{r^2(x, a_x)}}{(1-\alpha)\oldpolicy(a_x|x) + \alpha\evalpolicy(a_x|x)}
\right] -
R_{\targetpolicy}^2}.$$

The partial derivative of this with respect to $\evalpolicy(a_x|x)$ is
$$\dfrac{\partial}{\partial \evalpolicy(a_x|x)}\text{Var}\left[\hat R^{\text{BAL}}_{\targetpolicy}\right]
= -\alpha\dfrac{\Ex\sqaren{r^2(x, a_x)}}{\paren{(1-\alpha)\oldpolicy(a_x|x) + \alpha\evalpolicy(a_x|x)}^2}\Pr(x).$$

Since $\dfrac{\Ex\sqaren{r^2(x, a_x)}}{\paren{(1-\alpha)\oldpolicy(a_x|x) + \alpha\evalpolicy(a_x|x)}^2}\Pr(x)$ is always positive, the partial derivative is always negative, and so all increases of $\evalpolicy(a_x|x)$ decrease the variance. This means that the variance-minimizing strategy is to set $\evalpolicy(a_x|x)$ to the maximum value of $1$. Therefore, for every $x\in\calX$ $\evalpolicy(a_x|x)=1=\targetpolicy(a_x|x)$, so the two policies are equal, as desired.
\end{proof}

\subsection{Proof of Proposition \ref{thm:varunbounded}, No Upper Bound on Variance Reductions}

Recall Proposition \ref{thm:varunbounded}, which states that there is no upper bound on the variance decrease attained by adding a single augmentation sample.

\begin{proof}
\label{proof:varunbounded}
A basic argument is simply to notice that an action, context pair's contribution to the variance for a given logging policy is roughly $\frac{1}{\oldpolicy(a|x)}\paren{\targetpolicy^2(a|x)\expectedsquarer}$. Since the last term is always positive, we can roughly note that changing from a logging policy where $\oldpolicy(a|x) = \epsilon$ to a balanced policy where $\blendedpolicy(a|x) = \frac{N\epsilon + 1}{N}$ results in roughly a $\paren{\frac{1}{\epsilon} - \frac{N}{N\epsilon + 1}}\paren{\targetpolicy^2(a|x)\expectedsquarer} = \frac{1}{\epsilon\paren{N \epsilon + 1}}\paren{\targetpolicy^2(a|x)\expectedsquarer}$ variance improvement. As $\epsilon \to 0$, this expression goes to $\infty$.

In fact, we will show that this holds for a somewhat more complex comparison, between the variance of getting one additional data point from the old logging policy versus getting one additional data point from a well-designed augmentation policy for the balanced estimator.

Recall the variance of  the IPS estimator:
$$
\text{Var}\left[\hat R^{\text{IPS}}_{\targetpolicy}\right]
= \dfrac{1}{N}\paren{\mathbb{E}_{x}\left[\sum_{a \in \calA}
\dfrac{\targetpolicy^2(a|x)\expectedsquarer}{\oldpolicy(a|x)}
\right] -
R_{\targetpolicy}^2}.
$$

Adding $1$ additional sample would decrease the variance to 
$$
\text{Var}\left[\hat R^{\text{IPS}}_{\targetpolicy}\right]
= \dfrac{1}{N+1}\paren{\mathbb{E}_{x}\left[\sum_{a \in \calA}
\dfrac{\targetpolicy^2(a|x)\expectedsquarer}{\oldpolicy(a|x)}
\right] -
R_{\targetpolicy}^2}.
$$

This is the same variance decrease that would be achieved by using the old policy $\oldpolicy$ as the augmentation policy. To keep the two variances straight between using the $\oldpolicy$ as the augmentation policy and using $\evalpolicy$ as the augmentation policy, we refer to the first variance as $\text{Var}\left[\hat R^{\text{IPS}}_{\targetpolicy}\right]$ and the second as $\text{Var}\left[\hat R^{\text{BAL}}_{\targetpolicy}\right]$.

If we define $V = \mathbb{E}_{x}\left[\sum_{a \in \calA}
\dfrac{\targetpolicy^2(a|x)\expectedsquarer}{\oldpolicy(a|x)}
\right] -
R_{\targetpolicy}^2$, then the decrease is as follows:

$$\dfrac{1}{N}V - \dfrac{1}{N+1}V = \dfrac{N+1 - N}{N(N+1)}V = \dfrac{1}{N(N+1)}V.$$

This means that as $N$ gets larger and larger, adding an additional datapoint will decrease the variance by less and less.

In contrast, recall the variance of the balanced estimator:
$$
\text{Var}\left[\hat R^{\text{BAL}}_{\targetpolicy}\right]
= \dfrac{1}{N}\paren{\mathbb{E}_{x}\left[\sum_{a \in \calA}
\dfrac{\targetpolicy^2(a|x)\expectedsquarer}{\blendedpolicy(a|x)}
\right] -
R_{\targetpolicy}^2}.
$$

To simplify our notation, define the following two terms:
\begin{align*}
    Y_{x, a} &= \targetpolicy^2(a|x)\expectedsquarer\\
    V_{x, a} &= \dfrac{\targetpolicy^2(a|x)\expectedsquarer}{\oldpolicy(a|x)}
\end{align*}

The improvement for getting an additional sample from our method as opposed to continuing with the original logging policy is as follows:

\begin{align*}
\text{Var}\left[\hat R^{\text{IPS}}_{\targetpolicy}\right] - \text{Var}\left[\hat R^{\text{BAL}}_{\targetpolicy}\right] &=\dfrac{1}{N+1}\paren{\mathbb{E}_{x}\left[\sum_{a \in \calA}
\dfrac{\targetpolicy^2(a|x)\expectedsquarer}{\oldpolicy(a|x)}
\right] -
\bcancel{R_{\targetpolicy}^2}}\\
&-\dfrac{1}{N+1}\paren{\mathbb{E}_{x}\left[\sum_{a \in \calA}
\dfrac{\targetpolicy^2(a|x)\expectedsquarer}{\blendedpolicy(a|x)}
\right] -
\bcancel{R_{\targetpolicy}^2}}\\
&= \dfrac{1}{N+1}\paren{\mathbb{E}_{x}\left[\sum_{a \in \calA}
Y_{x,a}\paren{\dfrac{1}{\oldpolicy(a|x)} - \dfrac{1}{\blendedpolicy(a|x)}}
\right]}
\end{align*}

Recall that $\blendedpolicy(a|x) = (1-\alpha)\oldpolicy(a|x) - \alpha\evalpolicy(a|x)$, and that $\alpha = \frac{\evaln}{\evaln+\oldn}$, or in this case $1/(N+1)$. We can establish a lower bound for the variance improvement by restricting the class that the minimization is performed over to functions $\delta_a: \delta_{a'}(a|x) = 1$ when $a = a'$, and $\delta_{a'}(a|x)=0$ when $a \neq a'$.

\begin{align*}
\dfrac{1}{\oldpolicy(a|x)} - \dfrac{1}{\blendedpolicy(a|x)} &= \dfrac{1}{\oldpolicy(a|x)} - \dfrac{1}{\frac{N}{N+1}\oldpolicy(a|x) + \frac{1}{N+1}\delta_{a'}(a|x)}\\
&= \dfrac{1}{\oldpolicy(a|x)} - \dfrac{N+1}{N\oldpolicy(a|x) + \delta_{a'}(a|x)}\\
&= \dfrac{N\oldpolicy(a|x) + \delta_{a'}(a|x) - (N+1)\oldpolicy(a|x)}{\oldpolicy(a|x)\paren{N\oldpolicy(a|x) + \delta_{a'}(a|x)}}\\
&= \dfrac{\delta_{a'}(a|x) -\oldpolicy(a|x)}{\oldpolicy(a|x)\paren{N\oldpolicy(a|x) + \delta_{a'}(a|x)}}\\
\end{align*}

Substituting that back into the previous equation for the variance decrease, and because $ V_{x,a} = \dfrac{Y_{x, a}}{\oldpolicy(a|x)}$.

\begin{align*}
\text{Var}\left[\hat R^{\text{IPS}}_{\targetpolicy}\right] - \text{Var}\left[\hat R^{\text{BAL}}_{\targetpolicy}\right]
&= \dfrac{1}{N+1}\paren{\mathbb{E}_{x}\left[\sum_{a \in \calA}
Y_{x,a}\paren{\dfrac{\delta_{a'}(a|x) -\oldpolicy(a|x)}{\oldpolicy(a|x)\paren{N\oldpolicy(a|x) + \delta_{a'}(a|x)}}}\right]}\\
&= \dfrac{1}{N+1}\paren{\mathbb{E}_{x}\left[\sum_{a \in \calA}
V_{x,a}\paren{\dfrac{\delta_{a'}(a|x) -\oldpolicy(a|x)}{N\oldpolicy(a|x) + \delta_{a'}(a|x)}}
\right]}
\end{align*}

Note that whenever $a'\neq a$, $\delta_{a'}(a|x) = 0$. This means that

$$\forall a\neq a': V_{x,a}\paren{\dfrac{\delta_{a'}(a|x)-\oldpolicy(a|x)}{N\oldpolicy(a|x) + \delta_{a'}(a|x)}} = V_{x,a}\paren{\dfrac{0-\oldpolicy(a|x)}{N\oldpolicy(a|x) + 0}} = \dfrac{1}{N}V_{x, a}.$$

This allows us to simplify the term inside the expectation.

\begin{align*}
\sum_{a \in \calA}
V_{x,a}\paren{\dfrac{\delta_{a'}(a|x) -\oldpolicy(a|x)}{N\oldpolicy(a|x) + \delta_{a'}(a|x)}}
&= 
\paren{\dfrac{1 -\oldpolicy(a'|x)}{N\oldpolicy(a'|x) + 1}}V_{x,a'}
- \dfrac{1}{N} \sum_{a\neq a' \in \calA}
V_{x,a}
\end{align*}

Substituting this back in, now have a lower bound for the variance decrease:
$$\text{Var}\left[\hat R^{\text{IPS}}_{\targetpolicy}\right] - \text{Var}\left[\hat R^{\text{BAL}}_{\targetpolicy}\right] \geq \dfrac{1}{N+1}\paren{\mathbb{E}_{x}\left[
\paren{\dfrac{1 -\oldpolicy(a'|x)}{N\oldpolicy(a'|x) + 1}}V_{x,a'}
- \dfrac{1}{N} \sum_{a\neq a'}
V_{x,a}
\right]}$$

Choosing $a' = \argmax_{a\in\calA}V_{x, a}$ for each $x$, we can set $V_{x,a'} \geq V_{x, a}$ for all other $a$. This changes our lower bound to the following, since $V_{x, a'} \geq  V_{x, a}$ implies $-V_{x, a'} \leq -V_{x, a}$.

\begin{align*}
\text{Var}\left[\hat R^{\text{IPS}}_{\targetpolicy}\right] - \text{Var}\left[\hat R^{\text{BAL}}_{\targetpolicy}\right] &\geq \dfrac{1}{N+1}\paren{\mathbb{E}_{x}\left[
\paren{\dfrac{1 -\oldpolicy(a'|x)}{N\oldpolicy(a'|x) + 1}}V_{x,a'}
- \dfrac{1}{N}\sum_{a\neq a'}
V_{x,a'}
\right]}\\
&= \dfrac{1}{N+1}\paren{\mathbb{E}_{x}\left[
\paren{\dfrac{1 -\oldpolicy(a'|x)}{N\oldpolicy(a'|x) + 1}}V_{x,a'}
- \dfrac{|\calA| - 1}{N}
V_{x,a'}
\right]}\\
&= \dfrac{1}{N+1}\paren{\mathbb{E}_{x}\left[
\paren{\dfrac{1 -\oldpolicy(a'|x)}{N\oldpolicy(a'|x) + 1} - \dfrac{|\calA| - 1}{N}}V_{x,a'}\right]}\\
&= \dfrac{1}{N+1}\paren{\mathbb{E}_{x}\left[
\dfrac{N -N\oldpolicy(a'|x) - \paren{N\oldpolicy(a'|x) + 1}\paren{|\calA| - 1}}{N\paren{N\oldpolicy(a'|x) + 1}}V_{x,a'}\right]}\\
&= \dfrac{1}{N+1}\paren{\mathbb{E}_{x}\left[
\dfrac{
N - \bcancel{N\oldpolicy(a'|x)} - N\oldpolicy(a'|x)|\calA| + \bcancel{N\oldpolicy(a'|x)} - |\calA|  +1
}{N\paren{N\oldpolicy(a'|x) + 1}}V_{x,a'}\right]}
\end{align*}

Noting that $V_{x, a'} \geq 0$, we can solve for what $\oldpolicy(a'|x)$ makes the whole term positive.

\begin{align*}
\dfrac{1 -\oldpolicy(a'|x)}{N\oldpolicy(a'|x) + 1} &\geq \dfrac{|\calA| - 1}{N} \\
N(1 -\oldpolicy(a'|x)) &\geq (|\calA| - 1)(N\oldpolicy(a'|x) + 1) \\
N &\geq N|\calA|\oldpolicy(a'|x)  +  (|\calA| - 1)\\
\dfrac{N - |\calA| + 1}{N|\calA|}  &\geq \oldpolicy(a'|x)
\end{align*}

Note that the easiest way to maximize $V_{x, a}$ is to choose an $a'$ such that $\oldpolicy(a'|x)$ is smaller. As $\oldpolicy(a'|x)$ shrinks, the above equation demonstrates that the $\paren{\dfrac{1 -\oldpolicy(a'|x)}{N\oldpolicy(a'|x) + 1} - \dfrac{|\calA| - 1}{N}}V_{x,a'}$ term becomes positive, and therefore the variance decrease becomes positive.

Further, note that this bound has a fair amount of slack -- it only considers deterministic policies $\delta_a$ to sample from, and it simplifies the math by assuming that each action gets the maximum possible variance decrease in going from $\frac{1}{N} V_{x, a}$ to $\frac{1}{N+1} V_{x, a}$.

This shows that if $\oldpolicy(a'|x)$ is small enough, then $V_{x, a'}$ is multiplied by a positive factor. However, as $\oldpolicy(a'|x)$ gets smaller, $V_{x, a'} =  \dfrac{\targetpolicy^2(a'|x)\expectedsquarer}{\oldpolicy(a'|x)}$ can become arbitrarily large. This resulting in an arbitrarily large $\text{Var}\left[\hat R^{\text{IPS}}_{\targetpolicy}\right] - \text{Var}\left[\hat R^{\text{BAL}}_{\targetpolicy}\right]$, and therefore an arbitrarily large variance decrease is possible.
\end{proof}

\end{document}